\documentclass[conference]{IEEEtran}

\usepackage{mathtools, amsthm, amsopn, amsmath}
\usepackage{etex}
\usepackage{spconf,graphicx}
\usepackage{endnotes}
\usepackage{hyperref}
\usepackage{epsfig,psfrag}
\usepackage{pst-all}
\usepackage{amssymb,amsfonts,upref,cite,epsf,color,bm}
\usepackage{graphicx}
\usepackage{color}
\usepackage{calc}
\usepackage{booktabs}
\usepackage{tikz}
\usepackage{pgfplots}
\usepackage{subfig}
\usepackage{csvsimple}

\usetikzlibrary{arrows, calc}
\usetikzlibrary{positioning}
\usetikzlibrary{decorations}
\usetikzlibrary{decorations.pathreplacing, calligraphy}

\newcommand\seednodes{\mathcal{S}}

\usepackage{algorithm, float}
\usepackage{algpseudocode}
\floatname{algorithm}{Algorithm}
\algnewcommand\algorithmicinput{\textbf{Input:}}
\algnewcommand\INPUT{\item[\algorithmicinput]}
\algnewcommand\algorithmicoutput{\textbf{Output:}}
\algnewcommand\OUTPUT{\item[\algorithmicoutput]}

\usepackage[symbols,toc,automake,acronym]{glossaries-extra}

\makeglossaries

\newglossaryentry{latex}
{
    name=latex,
    description={LaTeX (short for Lamport TeX) is a document preparation system. The user has to think 
    	about only the content to put in the document and the software will take care of the formatting. }
}

\newglossaryentry{glsy}
{
    name=glossary,
    description={Acronyms and terms which are generally unknown or new to common readers.}
}

\newglossaryentry{minimum}
{
	name=minimum,
	description={Given a set of real numbers, the minimum is the smallest of those numbers.}
}

\newglossaryentry{maximum}
{
	name=maximum,
	description={Given a set of real numbers, the maximum is the largest of those numbers.}
}

\newglossaryentry{mean}
{
	name=mean,
	description={The expectation of a real-valued random variable.}
}

\newglossaryentry{variance}
{
	name={variance},
	description={The variance $V$ of a real-valued \gls{rv} $x$ is defined as the expectation $\expect\big\{ \big( x - \expect\{x \} \big)^{2} \big\}$ 
		of the squared difference $x$ and its expectation $\expect\{x \}$. We extend this definition to vector-valued \gls{rv}s $\featurevec$ 
		as $\expect\big\{ \big\| \featurevec - \expect\{\featurevec \} \big\|_{2}^{2} \big\}$.} ,first={variance},text={variance} 
}

\newglossaryentry{nn}
{
	name={nearest neighbour},
	description={Nearest neighbour methods learn a hypothesis $h: \featurespace \rightarrow \labelspace$ whose 
		function value $h(\featurevec)$ is solely detemined by the nearest neighbours in the 
		feature space $\featurespace$ },
	first={nearest neighbour (NN)},text={NN} 
}

\newglossaryentry{bias}
{
	name={bias},
	description={Consider some unknown quantity $\bar{\weight}$, e.g., the true weight in a linear model $\truelabel = \bar{\weight} \feature + e$ 
		relating feature and label of a \gls{datapoint}. We might use an ML method (e.g., based on \gls{erm}) to 
		compute an estimate $\hat{\weight}$ for the $\bar{\weight}$ based on a set of \gls{datapoint}s that are 
		realizations of \gls{rv}s. The (squared) bias incurred by the estimate $\hat{\weight}$ is typically defined as 
		$\biasterm^{2} \defeq \big( \expect \{ \hat{\weight}  \}- \bar{\weight}\big)^{2}$. We extend this definition to 
		vector-valued quantities using the squared Euclidean norm $\biasterm^{2} \defeq \big\| \expect \{ \widehat{\weights}  \}- \overline{\weights}\big\|_{2}^{2}$.},
first={bias},text={bias} 
}

\newglossaryentry{classification}
{
	name={classification},
	description={Classification is the task of determining a discrete-valued label $\truelabel$ of a \gls{datapoint} 
		based solely on its features $\featurevec$. The label $\truelabel$ belongs to a finite set, such as $\truelabel \in \{ -1,1\}$, 
		or $\truelabel \in \{1,\ldots,19\}$ and represents a category to which the corresponding \gls{datapoint} belongs to.},first={classification},text={classification} 
}

\newglossaryentry{condnr}
{
	name={condition number},
	description={The condition number $\kappa(\mathbf{Q})$ of a \gls{psd} matrix $\mathbf{Q}$ is 
		the ratio of the largest to the smallest eigenvalue of $\mathbf{Q}$.},first={condition number},text={condition number} 
}

\newglossaryentry{classifier}
{
	name={classifier},
	description={A classifier is a hypothesis $h(\featurevec)$ that is used to predict a discrete-valued label. 
		Strictly speaking, a classifier is a hypothesis $h(\featurevec)$ that can take only a finite number of 
		different values. However, we are sometimes sloppy and use the term classifier also for a hypothesis 
		that delivers a real number which is thresholded to obtain the predicted label value. 
		For example, in a binary classification problem with label values $\truelabel \in \{ -1,1\}$, we refer to a 
		linear hypothesis $h(\featurevec) =\weights^{T}\featurevec$ as classifier if it is used to predict the label 
		value according to $\hat{\truelabel} = 1$ when $h(\featurevec) \geq 0$  and $\hat{\truelabel}=-1$ otherwise.},first={classifier},text={classifier} 
}

\newglossaryentry{emprisk}
{name={empirical risk},
  description={The empirical risk of a given hypothesis on a given set of datapoints is the average loss of the hypothesis computed over all datapoints in that set.},first={empirical risk},text={empirical risk} 
}

\newglossaryentry{risk}
{name={risk},
	description={Consider a hypothesis $h$ that is used to predict the label $\truelabel$ of a \gls{datapoint} based on its features $\featurevec$. 
	We measure the quality of a particular prediction using a \gls{lossfunc} $\lossfun$. 
	The risk of a hypothesis is defined as the expected loss when it is applied to a random \gls{datapoint} $\big(\featurevec,\truelabel\big)$ with 
 joint probability distribution $p(\featurevec,\truelabel)$. Note that the risk of $h$ depends on both, the specific choice for the \gls{lossfunc} and 
the joint proability distribution.},
	first={risk},text={risk} 
}

\newglossaryentry{actfun}
{name={activation function},
	description={Each artificial neuron within an \gls{ann} consists of an activation function that maps 
		the inputs of the neuron to a single output value. In general, an activation function is a non-linear 
		map of the weighted sum of neuron inputs (this weighted sum is the activation of the neuron).},
first={activation function},text={activation function} 
}

\newglossaryentry{deepnet}
{name={deep net},
	description={We refer to an \gls{ann} with a (relatively) large number of hidden layers as a deep \gls{ann} or ``deep net''. Deep nets 
	are used to represent the \gls{hypospace}s of deep learning methods \cite{Goodfellow-et-al-2016}.},
	first={deep \gls{ann} (deep net)},text={deep net} 
}

\newglossaryentry{baseline}
{name={baseline},
	description={A reference value or benchmark for the average loss incurred by a hypothesis when applied to the \gls{datapoint}s 
		generated in a specific ML application. Such a reference value might be obtained from human performance (e.g., error rate 
		of dermatologists diagnosing cancer from visual inspection of skin areas) or other ML methods (``competitors'')}, 
	first={baseline},text={baseline} 
}

\newglossaryentry{esterr}
{name={estimation error},
	description={Consider \gls{datapoint}s with feature vectors $\featurevec$ and label 
		$\truelabel$. In some applications we can model the relation between features and label of a \gls{datapoint} 
		as $\truelabel = \bar{h}(\featurevec) + \varepsilon$. Here we used some true hypothesis $\bar{h}$ and a noise 
		term $\varepsilon$ which might represent modelling or labelling errors. The estimation error incurred by a ML 
		method that learns a hypothesis $\hat{h}$, e.g., using \gls{erm}, is defined as $\hat{h} - \bar{h}$. 
		For a parametrized hypothesis space, consisting of hypothesis maps that are determined by a parameter vector $\weights$, 
		we define the estimation error in terms of parameter vectors as $\Delta \weights = \widehat{\weights} - \overline{\weights}$.}
	first={estimation error},text={estimation error} 
}

\newglossaryentry{msee}
{name={mean squared estimation error},
	description={Consider a ML method that uses a parametrized hypothesis space. For a given \gls{trainset}, whose \gls{datapoint}s are 
		interpreted as realizations of \gls{rv}s, the ML method learns the parameters incurring the \gls{esterr} $\Delta \weights$.   
		The mean squared estimation error is defined as the expectation $\expect \big\{ \big\| \Delta \weights \big\|^{2} \big\} $ 
		of the squared Euclidean norm of the \gls{esterr}.},
	first={mean squared estimation error (MSEE)},text={MSEE} 
}

\newglossaryentry{expert}
{name={expert},
	description={ML aims at learning a hypothesis $h$ that accurately predicts the label 
		of a \gls{datapoint} based on its features. We measure the prediction error using 
		some \gls{lossfunc}. Ideally we want to find a hypothesis that incurres minimum loss. 
		One approach to make this goal precise is to use the \gls{iidasspt} and use the resulting 
		\gls{bayesrisk} as the benchmark level for the (average) loss of a hypothesis. Alternatively 
		we might know a reference or benchmark hypothesis $h'$ which might be obtained by 
		some existing ML mehtod. We can then compare the loss incurred by $h$ with the loss 
		incurred by $h'$. Such a reference or baseline hypothesis $h'$ is refered to as an \gls{expert}. 
		Note that an expert might deliver very poor predictions. We typically compare against many 
	   different experts and aim at incurring not much more loss than the best among those experts (this is 
	   known as regret minimization) \cite{PredictionLearningGames,HazanOCO}.}
	first={expert},text={expert} 
}

\newglossaryentry{regret}
{name={regret},
	description={The regret of a hypothesis $h$ relative to another hypothesis $h'$, which serves as a reference of baseline, 
		      is the difference between the \gls{loss} incurred by $h$ and the \gls{loss} incurred by $h'$ \cite{PredictionLearningGames}. 
		      The baseline hypothesis $h'$ is also refered to as an \gls{expert}.}
	first={regret},text={regret} 
}

\newglossaryentry{differentiable}
{name={differentiable},
	description={A function $f: \mathbb{R}^{\featuredim} \rightarrow \mathbb{R}$  is differentiable if it 
		has a \gls{gradient} $\nabla f ( \mathbf{x})$ everywhere (for every $\mathbf{x} \in \mathbb{R}^{\featuredim}$).},
	first={differentiable},text={differentiable} 
}

\newglossaryentry{gradient}
{name={gradient},
	description={For a real-valued function $f: \mathbb{R}^{\featuredim} \rightarrow \mathbb{R}: \weights \mapsto f(\weights)$, 
	a vector $\va$ such that $\lim_{\weights \rightarrow \weights'} \frac{f(\weights) - \big(f(\weights')+ \va^{T} (\weights- \weights') \big) }{\| \weights-\weights'\|}=0$ is 
	referred to as the gradient of $f$ at $\weights'$. If such a vector exists it is denoted $\nabla f(\weights')$ or $\nabla f(\weights)\big|_{\weights'}$  .},
	first={gradient},text={gradient} 
}

\newglossaryentry{subgradient}
{name={subgradient},
description={For a real-valued function $f: \mathbb{R}^{\featuredim} \rightarrow \mathbb{R}: \weights \mapsto f(\weights)$, 
		a vector $\va$ such that $f(\weights) \geq  f(\weights') +\big(\weights-\weights' \big)^{T} \va$ is 
		referred to as a subgradient of $f$ at $\weights'$.},
	first={subgradient},text={subgradient} 
}

\newglossaryentry{relu}
{name={ReLU},
	description={The rectified linear unit or ``ReLU'' is a popular choice for the \gls{actfun} of a neuron within an \gls{ann}. 
	It is defined as $g(z) = \max\{0,z\}$ with $z$ being the weighted input of the neuron.}, first = {rectified linear unit (ReLU)}, text={ReLU} 
}

\newglossaryentry{hypothesis}
{name={hypothesis},
	description={A map (or function) $h: \featurespace \rightarrow \labelspace$ from the 
		feature space $\featurespace$ to the label space $\labelspace$. 
		Given a \gls{datapoint} with features $\featurevec$ we use a hypothesis map 
		to estimate (or approximate) the label $y$ using the predicted label $\hat{\truelabel} = h(\featurevec)$. 
		ML is about learning (finding) a hypothesis map such that $\hat{\truelabel} \approx h(\vx)$ for any \gls{datapoint}.},
	first={hypothesis},text={hypothesis}  
}

\newglossaryentry{vcdim}
{name={Vapnik–Chervonenkis (VC) dimension},
	description={The VC dimension is maybe the most widely used concept for measuring the size of infinite \gls{hypospace}s. 
		For a precise definition of the VC dimension and discussion of its applications 
		in ML we refer to \cite{ShalevMLBook}. },
	first={Vapnik–Chervonenkis (VC) dimension},text={VC dimension}  
}

\newglossaryentry{effdim}
{name={effective dimension},
	description={The effective dimension $\effdim{\hypospace}$ of an infinite \gls{hypospace} $\hypospace$ 
		is a measure of its size. Roughly speaking the effective dimension is equal to the number of independent 
		tunable parameters of the model. These parameters might be the weights used by linear map or 
		the weights and bias terms of an \gls{ann}.},
	first={effective dimension},text={effective dimension}  
}

\newglossaryentry{labelspace}
{name={label space},
	description={Consider a ML application that involves \gls{datapoint}s characterized by features 
		and labels. The label space of a given ML application or method is constituted by all potential 
		values that the label of a \gls{datapoint} can take on. A popular choice for the label space in 
		regression problems (or methods) is $\labelspace = \mathbb{R}$. Binary classification problems 
		(or methods) use a label space that consists of two different elements, e.g., $\labelspace =\{-1,1\}$, $\labelspace=\{0,1\}$ 
		or $\labelspace = \{ \mbox{``cat image''}, \mbox{''no cat image''} \}$  }, first={label space},text={label space}  
}

\newglossaryentry{bootstrap}
{name={bootstrap},
	description={Consider a probabilistic model that interprets a given set of \gls{datapoint}s $\dataset = \big\{ \datapoint^{(1)},\ldots,\datapoint^{(\samplesize)}\big\}$ as realizations of \gls{iid} \gls{rv}s with a common probability distribution $p(\datapoint)$. 
	The bootstrap uses the histogram of $\dataset$ as the underlying proability distribution $p(\datapoint)$. 
	},
	first={bootstrap},text={bootstrap}  
}

\newglossaryentry{featurespace}
{name={feature space},
	description={
		The feature space of a given ML application or method is constituted by all potential 
		values that the feature vector of a \gls{datapoint} can take on. Within this book the most 
		frequently used choice for the feature space is the \gls{euclidspace} $\mathbb{R}^{\featuredim}$ 
		with dimension $\featurelen$ being the number of individual features of a \gls{datapoint}.},
	first={feature space},text={feature space}  
}

\newglossaryentry{missingdata}
{name={missing data},
	description={By missing data, we refer to a situation where some feature values of a 
		subset of \gls{datapoint}s are unknown. Data imputation techniaues aim at estimating 
		(predicting) these missing feature values \cite{Abayomi2008DiagnosticsFM}. },
	first={missing data},text={missing data}  
}

\newglossaryentry{psd}
{name={positive semi-definite},
	description={A symmetric matrix $\mQ = \mQ^{T} \in \mathbb{R}^{\featuredim \times \featuredim}$ 
		is referred to as positive semi-definite if $\vx^{T} \mQ \vx \geq 0$ for every vector $\vx \in \mathbb{R}^{\featuredim}$.},
	first={positive semi-definite (psd)},text={psd}  
}

\newglossaryentry{features}
{name={features},
	description={Features are those properties of a \gls{datapoint} that can be measured or computed in an 
		automated fashion. For example, if a \gls{datapoint} is a bitmap image, then we could use 
		the red-green-blue intensities of its pixels as features. Some widely used synonyms for 
		the term feature are ``covariate'',``explanatory variable'', ``independent variable'', ``input (variable)'', ``predictor (variable)'' or ``regressor''  \cite{Gujarati2021,Dodge2003,Everitt2022}. 
		However, this book makes consquent use of the term features for the low-level properties of \gls{datapoint}s that 
		can be measured eaisily.}, first={features},text={features}  
}

\newglossaryentry{label}
{name={label},
	description={A higher level fact or quantity of interest associated with a \gls{datapoint}. If a \gls{datapoint} is an image, 
		its label might be the fact that it shows a cat (or not). Some widely used synonyms for 
		the term label are "response variable", "output variable" or "target"  \cite{Gujarati2021,Dodge2003,Everitt2022}.
 },
	first={label},text={label}  
}

\newglossaryentry{noniid}
{name={non-i.i.d.},
	description={See \gls{noniiddata}.},first={non-i.i.d.},text={non-i.i.d.}  
}

\newglossaryentry{data}
{name={data},
	description={A set of \gls{datapoint}s.},first={data},text={data}  
}

\newglossaryentry{dataset}
{name={dataset},
	description={With a slight abuse of notation we use the terms ``dataset`` or ``set of datapoints'' to 
		refer to an indexed list of \gls{datapoint}s $\datapoint^{(1)},\ldots,$. Thus, there is a first \gls{datapoint} 
		$\datapoint^{(1)}$, a second \gls{datapoint} $\datapoint^{(2)}$ and so on. Strictly speaking a dataset is a list 
		and not a set \cite{HalmosSet}. By using indexed lists of \gls{datapoint}s we avoid some of the challenges 
		arising in concept of an abstract set.},first={dataset},text={dataset}  
}

\newglossaryentry{predictor}
{name={predictor},
	description={A predictor is a \gls{hypothesis} whose function values are numeric, such as real numbers. 
		Given a \gls{datapoint} with features $\featurevec$, the predictor value $h(\featurevec) \in \mathbb{R}$ 
		is used as a prediction (estimate/guess/approximation) for the true numeric label $\truelabel \in \mathbb{R}$ of the \gls{datapoint}. },first={predictor},text={predictor}  
}

\newglossaryentry{labeled datapoint}
{name={labeled datapoint},
 description={A \gls{datapoint} whose label is known or has been determined by some means (might require human experts).},
 first={labeled datapoint},text={labeled datapoint}  
}

\newglossaryentry{rv}
{name={random variable},
 description={Formally, a random variable is a map from a set of elementary events to a set of values. The set of 
 	elementary events is equipped with a probability measure. A real-valued random variable maps elementary events 
 	to real numbers $\mathbb{R}$. A discrete random variable maps elementary events to a finite set such as $\{-1,1\}$ or $\{ \mbox{cat}, \mbox{no cat} \}$. A vector-valued random variable maps elementary events to the \gls{euclidspace} $\mathbb{R}^{\featuredim}$ 
 	with some dimension $\featuredim \in \mathbb{N}$.}, first={random variable (RV)},text={RV}  
}

\newglossaryentry{trainset}
{
	name={training set},
	description={A set of data points that is used in \gls{erm} to train a hypothesis $\hat{h}$. The average 
		loss of $\hat{h}$ on the training set is referred to as the training error. The comparison between training and 
		validation error informs adaptations of the ML method (such as using a different \gls{hypospace}).},first={training set},text={training set}  
}

\newglossaryentry{trainerr}
{
	name={training error},
	description={Consider a ML method that aims at learning a hypothesis $h \in \hypospace$ out of a \gls{hypospace}. 
		We refer to the average loss or \gls{emprisk} of a hypothesis $h \in \hypospace$ on a dataset as training error if it 
		is used to choose between different hypotheses. The principle of \gls{erm} is find the hypothesis $h^{*} \in \hypospace$ 
		with smallest training error. Overloading the notation a bit, we might refer by training error also to the minimum \gls{emprisk} 
		achieved by the optimal hypothesis $h^{*} \in \hypospace$.},first={training error},text={training error}  
}

\newglossaryentry{datapoint}
{
	name={data point},
	description={A \gls{datapoint} is any object that conveys information \cite{coverthomas}. Data points might be 
		students, radio signals, trees, forests, images, \gls{rv}s, real numbers or proteins. We characterize data points 
		using two types of properties. One type of property is referred to as a feature. \Gls{features} are properties of a 
		\gls{datapoint} that can be measured or computed in an automated fashion. Another type of property is referred to as a 
		\gls{label}. The label of a \gls{datapoint} represents a higher-level facts or quantities of interest. In contrast to \gls{features}, 
		determining the label of a \gls{datapoint} typically requires human experts (domain experts). Roughly speaking, ML aims 
		at predicting the label of a \gls{datapoint} based solely on its features.},first={data point},text={data point}  
}

\newglossaryentry{valerr}
{name={validation error},
 description={Consider a hypothesis $\hat{h}$ which is obtained by \gls{erm} on a \gls{trainset}. The 
 average loss of $\hat{h}$ on a \gls{valset}, which is different from the \gls{trainset}, is referred to as 
 the validation error.},first={validation error},text={validation error}  
}

\newglossaryentry{valset}
{name={validation set},
  description={A set of \gls{datapoint}s that has not been used as \gls{trainset} in \gls{erm} to train a hypothesis $\hat{h}$. 
  The average loss of $\hat{h}$ on the validation set is referred to as the validation error and used to diagnose 
  the ML method. The comparison between training and validation error informs adaptations of the 
  ML method (such as using a different \gls{hypospace}).},first={validation set},text={validation set}  
}

\newglossaryentry{testset}
{name={test set},
	description={A set of \gls{datapoint}s that have neither been used in a \gls{trainset} to learn parameters of a model 
		nor in a \gls{valset} to choose between different models (by comparing validation errors).},first={test set},text={test set}  
}

\newglossaryentry{linclass}{name={linear classifier}, description={A classifier $h(\featurevec)$ maps the 
		feature vector $\featurevec \in \mathbb{R}^{\featuredim}$ of a datapoint to a predicted label $\hat{\truelabel} \in \labelspace$ out of 
		a finite set of label values $\labelspace$. We can characterize such a classifier equivalently by the decision regions 
		$\decreg{a}$, for every possible label value $a \in \labelspace$. Linear classifiers are such that the boundaries 
		between the regions $\decreg{a}$ are hyperplanes in $\mathbb{R}^{\featuredim}$.  },first={linear classifier},text={linear classifier} }

\newglossaryentry{erm}{name={empirical risk minimization}, description={Empirical risk minimization is the optimization problem 
		of finding the hypothesis with minimum average loss (empirical risk) on a given set of \gls{datapoint}s (the \gls{trainset}).
		 Many ML methods are special cases of empirical risk minimization.},first={empirical risk minimization (ERM)},text={ERM} }

\newglossaryentry{multilabelclass}{name={multi-label classification}, description={Multi-label classification problems and methods involve 
	\gls{datapoint}s that are characterized by several individual labels.},first={multi-label classification},text={multi-label classification} }

\newglossaryentry{ssl}{name={semi-supervised learning}, description={Semi-supervised learning methods use (large amounts of) 
		unlabeled \gls{datapoint}s to support the learning of a hypothesis from (a small number of) labeled \gls{datapoint}s \cite{SemiSupervisedBook}. },first={semi-supervised learning (SSL)},text={SSL} }

\newglossaryentry{regularization}{name={regularization}, description={Regularization techniques modify the \gls{erm} 
		principle such that the learnt hypothesis performs well also outside the \gls{trainset} which is used in \gls{erm}. 
		One specific approach to regularization is by adding a penalty or regularization term to the 
		objective function of \gls{erm} (which is the average loss on the \gls{trainset}). This regulazation term 
	    can be interpreted as an estimate for the increase in the expected loss (risk) compared to the average loss on the \gls{trainset}. },first={regularization},text={regularization} }

\newglossaryentry{rerm}{name={regularized empirical risk minimization}, description={Regularized empirical risk minimization is the problem 
		of finding the hypothesis that optimally balances the average loss (empirical risk) on a \gls{trainset} with 
		a regularization term. The regularization term penalizes a hypothesis that is not robust against (small) perturbations 
		of the \gls{datapoint}s in the \gls{trainset}.},first={regularized empirical risk minimization (RERM)},text={RERM} }
	
\newglossaryentry{gtv}{name={generalized total variation}, description={Generalized total variation measures the changes of 
	vector-valued node attributes of a graph.},first={generalized total varation (GTV)},text={GTV} }
	
\newglossaryentry{srm}{name={structural risk minimization}, description={Structural risk minimization is the problem 
		of finding the hypothesis that optimally balances the average loss (empirical risk) on a \gls{trainset} with 
		a regularization term. The regularization term penalizes a hypothesis that is not robust against (small) perturbations 
	of the \gls{datapoint}s in the \gls{trainset}.},first={structural risk minimization (SRM)},text={SRM} }

\newglossaryentry{netexpfam}{name={networked exponential families}, description={A networked collection of 
		exponential families having a separaate parameter vector for each node of the network. These parameter 
		vectors are coupled via the network structure. },first={networked exponential family (nExpFam)},text={nExpFam} }

\newglossaryentry{scatterplot}{name={scatterplot}, description={A visualization technique that depicts data points by markers in a 
	two-dimensional plane.},first={scatterplot},text={scatterplot} }

\newglossaryentry{learnrate}{name={learning rate}, description={Consider an iterative method for finding or learning 
				a good choice for a hypothesis. Such an iterative method repeats similar computational (update) steps that adjust 
				or modify the current choice for the hypothesis to obtain an improved hypothesis. A prime example for such an 
				iterative learning method is \gls{gd} and its variants (see \ref{ch_GD}). We refer by learning rate to any parameter 
				of an iterative learning method that controls the extent by which the current hypothesis might be modified or 
				improved in each iteration. A prime example for such a parameter is the step size used in \gls{gd}. Within this book 
		we use the term learning rate mostly as a synonym for the step size of (a variant of) \gls{gd}},
	first={learning rate},text={learning rate} }

\newglossaryentry{featuremap}{name={feature map}, description={A map that transforms some raw features into a 
		new feature vector. The new feature vector might be preferable over the raw features for several reasons. It might be possible 
	to use linear hypothesis with the new feature vectors. Another reason could be that the new feature vector is much shorter and 
therefore avoids overfitting or can be used for a \gls{scatterplot}},
	first={feature map},text={feature map} }

\newglossaryentry{fl}{name={federated learning (FL)}, description={Federated learning is an umbrella term for ML methods that 
		train models in a collaborative fashion using decentralized data and computation.},
	first={federated learning (FL)},text={FL} }

\newglossaryentry{iid}{name={i.i.d.}, description={ independent and identically distributed; e.g., 
		``$x,y,z$ are i.i.d. RVs'' means that the joint probability distribution $p(x,y,z)$ of the RVs 
		$x,y,z$ factors into the product $p(x)p(y)p(z)$ of the marginal probability distributions of 
		the variables $x,y,z$ which are identical.},
	first={independent and identically distributed (i.i.d.)},text={i.i.d.} }

\newglossaryentry{outlier}{name={outlier}, description={Many ML methods are motivated by the \gls{iidasspt} which interprets 
data points as \gls{iid} ealizations of \gls{rv}s with a common probability distribution. The \gls{iidasspt} is typically useful when the statistical 
properties of the data generation process are stationary (time-invariant). In some applications the data consists of a majority 
of ``regular'' \gls{datapoint}s that conform with an \gls{iidasspt} and a small number of data points that have fundamentally different 
statistical properties compared to the bulk of regular \gls{datapoint}s. We refer to a \gls{datapoint} that substantially deviates 
from the statistical properties of the majority of \gls{datapoint}s as an outlier.},
	first={outliers},text={outliers} }

\newglossaryentry{decisionregion}{name={decision region}, description={Consider a hypothesis map $h$ that can 
		only take values from a finite set $\labelspace$. We refer to the set of features $\featurevec \in \featurespace$ 
		that result in the same output $h(\featurevec)=a$ as a decision region of the hypothesis $h$. },first={decision region},text={decision region} }

\newglossaryentry{decisionboundary}{name={decision region}, description={Consider a hypothesis map $h$ that reads in 
		a feature vector $\featurevec \in \mathbb{R}^{\featuredim}$ and delivers a value from a finite set $\labelspace$. 
		The decision boundary induced by $h$ is the set of vectors $\featurevec \in \mathbb{R}^{\featuredim}$ that lie between 
		different\gls{decisionregion}s. More precisely, a vector $\featurevec$ belongs to the decision boundary if and only if 
	each neighborhood $\{ \featurevec': \| \featurevec - \featurevec' \| \leq \varepsilon \}$, for any $\varepsilon >0$, contains 
  at least two vectors with different function values.},first={decision boundary},text={decision boundary} }

\newglossaryentry{euclidspace}{name={Euclidean space}, description={The Euclidean space $\mathbb{R}^{\featuredim}$ of dimension $\featuredim$ refers to the space of all vectors $\featurevec= \big(\feature_{1},\ldots,\feature_{\featurelen}\big)$, with real-valued entries $\feature_{1},\ldots,\feature_{\featuredim} \in \mathbb{R}$, 
whose geometry is defined by the inner product $\featurevec^{T} \featurevec' = \sum_{\featureidx=1}^{\featuredim} \feature_{\featureidx} \feature'_{\featureidx}$ between any two 
vectors $\featurevec,\featurevec' \in \mathbb{R}^{\featuredim}$ \cite{RudinBookPrinciplesMatheAnalysis}.},first={Euclidean space},text={Euclidean space} }

\newglossaryentry{eerm}{name={explainable empirical risk minimization}, description={An instance of structural risk minimization that adds a regularization term 
		to the training error in \gls{erm}. The regularization term is chosen to favour hypotheses that are intrinsically explainable for a user.},first={explainable empirical risk minimization (EERM)},text={EERM} }

\newglossaryentry{kmeans}{name={$k$-means}, description={The $k$-means algorithm is a hard 
		clustering method. It aims at assigning data points to clusters such that they have minimum 
		average distance from the cluster centre.},first={$k$-means},text={$k$-means} }

\newglossaryentry{xml}{name={explainable machine learning}, description={Explainable ML methods aim at 
		complementing predictions with explanations for how the prediction has been obtained.},first={explainable ML},text={explainable ML} }

\newglossaryentry{fmi}{name={Finnish Meteorological Institute}, description={The Finnish Meteorological Institute is a 
		government agency responsible for gathering and reporting weather data in Finland.},first={Finnish Meteorological Institute (FMI)},text={FMI} }

\newglossaryentry{highdimregime}{name={high-dimensional regime}, description={A ML method 
		or problem belongs to the high-dimensional regime if the \gls{effdim} of the \gls{model} is larger 
		than the number of available (labeled) \gls{datapoint}s. For example, linear regression belongs to the 
   	    high-dimensional regime whenever the number $\featuredim$ of features used to characterize datapoints is 
       larger than the number of \gls{datapoint}s in the \gls{trainset}. Another example for the high-dimensional 
       regime are deep learning methods that use a hypothesis space generated by a \gls{ann} with much more tunable 
       weights than the number of \gls{datapoint}s in the \gls{trainset}. The recent field of high-dimensional statistics 
       uses probability theory to analyze ML methods in the high-dimensional regime \cite{Wain2019,BuhlGeerBook}.},
   first={high-dimensional regime},text={high-dimensional regime} }

\newglossaryentry{gmm}{name={Gaussian mixture model}, description={Gaussian mixture models (GMM) are a family of 
		probabilistic models for  \gls{datapoint}s. Within a GMM, the feature vector $\featurevec$ of a \gls{datapoint} is 
		interpreted as being drawn from one out of $\nrcluster$ different multivariate normal (Gaussian) distributions 
		indexed by $\clusteridx=1,\ldots,\nrcluster$. The probability that the feature vector $\featurevec$ is drawn from 
		the $\clusteridx$-th Gaussian distribution is denoted $p_{\clusteridx}$. The GMM is parametrized by the probability 
		$p_{\clusteridx}$ of $\featurevec$ being drawn from the $\clusteridx$-th Gaussian distribution as well as the 
		mean vectors ${\bm \mu}^{(\clusteridx)}$ and covariance matrices $ {\bm \Sigma}^{(\clusteridx)}$ for $\clusteridx=1,\ldots,\nrcluster$. 
	 },first={Gaussian mixture model (GMM)},text={GMM} }
 
\newglossaryentry{ml}{name={maximum likelihood}, description={
		Consider \gls{datapoint}s that are interpreted as \gls{iid} realizations of \gls{rv}s with a common (but unknown) 
		probability distribution. Maximum likelihood methods find a parameter vector $\weights$ for a probabilistic 
		model $\prob{\datapoint; \weights}$ such that the probability (density) of observing the actucal data is maximized. 
		Loosely speaking, we try out all possible parameter vectors $\weights$ and determine the resulting 
		probability of observing the given datapoints if they would be \gls{iid} with common probability distribution $\prob{\datapoint; \weights}$. 
		The maximum likelihood estimator is the parameter vector that results in the highest probability (density). 
	},first={maximum likelihood},text={maximum likelihood}}

\newglossaryentry{em}{name={expectation maximization}, description={Expectation maximization is a generic 
		technique for estimating the parameters of a probabilistic model (a parametrized probability distribution) $\prob{\datapoint; \weights}$ 
		from data \cite{BishopBook,hastie01statisticallearning,GraphModExpFamVarInfWainJor}. Expectation maximization 
		delivers an approximation to the \gls{ml} estimate for the model parameters $\weights$. 
  },first={expectation maximization (EM)},text={EM}}

\newglossaryentry{ppca}{name={probabilistic prinicipal component analysis}, description={Probabilistic prinicipal 
		component analysis extends basic \gls{pca} by using a probabilistic model for \gls{datapoint}s. Within this 
		probabilistic model, the task of dimensionality reduction becomes an estimation problem that can be solved 
	using \gls{em} methods.},first={probabilistic prinicipal component analysis (PPCA)},text={PPCA}}

\newglossaryentry{linreg}{name={linear regression}, description={Linear regression aims at learning a 
		linear hypothesis map to predict a numeric \gls{label} based on numeric \gls{features} of a \gls{datapoint}. The quality 
	of a linear hypothesis map is typically measured using the average squared error loss incurred on a set 
	of labeled \gls{datapoint}s (the \gls{trainset}).},first={linear regression},text={linear regression}}

\newglossaryentry{logreg}{name={logistic regression}, description={Logistic regression aims at learning a 
		linear hypothesis map to predict a binary \gls{label} based on numeric \gls{features} of a \gls{datapoint}. 
		The quality of a linear hypothesis map (classifier) is measured using its average logistic loss on 
		some labeled datapoints (the \gls{trainset}).},first={logistic regression},text={logistic regression}}
	
\newglossaryentry{logloss}{name={logistic loss}, description={Consider a \gls{datapoint} that is characterized by the features $\featurevec$ 
		and a binary label $\truelabel \in \{-1,1\}$. We use a hypothesis $h$ to predict the label $\truelabel$ solely from the features $\featurevec$. 
		The logistic loss incurred by a specific hypothesis $h$ is defined as \eqref{equ_log_loss}.},first={logistic loss},text={logistic loss}}
	
\newglossaryentry{hingeloss}{name={hinge loss}, description={Consider a \gls{datapoint} that is characterized by a 
		feature vector $\featurevec \in \mathbb{R}^{\featuredim}$ and a binary label $\truelabel \in \{-1,1\}$. The hinge loss 
		incurred by a specific hypothesis $h$ is defined as \eqref{equ_hinge_loss}. A regularized variant of the hinge loss 
		is used by the \gls{svm} to learn a \gls{linclass} with maximum margin between the two classes (see Figure \ref{fig_svm}). 
	},first={hinge loss},text={hinge loss}}

\newglossaryentry{iidasspt}{name={i.i.d.\ assumption}, description={The i.i.d.\ assumption interprets \gls{datapoint}s of a \gls{dataset} 
		as the realizations of \gls{iid} \gls{rv}s.},first={i.i.d.\ assumption},text={i.i.d.\ assumption} }

\newglossaryentry{hypospace}{name={hypothesis space}, description={Every practical ML method uses a specific hypothesis space, 
		which we typically denote by $\hypospace$. The hypothesis space of a ML method is a subset of all possible maps 
		from the feature space to label space. The choice for the hypothesis space should take into account available 
		computational infrastructure of statistical aspects. If the computational infrastructure allows for efficient matrix 
		operations and we expect a linear relation between feature values and label, a good first candidate 
		for the hypothesis space is the space of linear maps.},first={hypothesis space},text={hypothesis space} }
	
\newglossaryentry{model}{name={model}, description={We use the term model as a synonym for \gls{hypospace}},first={model},text={model} }

\newglossaryentry{ai}{name={artificial intelligence}, description={Artificial intelligence aims at developing systems that behave 
		rational in the sense of maximizing a long-term reward.},first={artificial intelligence (AI)},text={AI} }
	
\newglossaryentry{clustering}{name={clustering}, description={Clustering refers to the problem ot determining for each \gls{datapoint} 
		within a \gls{dataset} to which cluster it belongs to. A cluster can be defined and represented in various ways, e.g., using 
		representative \gls{datapoint}s (``cluster means'') or an entire probability distribution (see \gls{gmm}).},first={clustering},text={clustering} }

\newglossaryentry{softclustering}{name={soft clustering}, description={Soft clustering methods determine, for each \gls{datapoint} within a dataset, 
		a soft cluster assignment or the degree of belonging to a particular cluster.},first={soft clustering},text={soft clustering} }

\newglossaryentry{huberloss}{name={Huber loss}, description={The Huber loss is a mixture of the squared error loss 
		and the absolute value of the prediction error.},first={Huber loss},text={Huber loss} }

\newglossaryentry{svm}{name={support vector machine}, description={A binary classification method for learning a 
		linear map that maximally seperates \gls{datapoint}s the two classes in the feature space (``maximum margin''). 
		Maximizing this separation is equivalent to minimizing a regularized variant of the \gls{hingeloss} \eqref{equ_hinge_loss}.},first={support vector machine (SVM)},text={SVM} }

\newglossaryentry{eigenvalue}{name={eigenvalue}, description={We refer to a number $\lambda \in \mathbb{R}$ 
		as eigenvalue of a square matrix $\mathbf{A} \in \mathbb{R}^{\featuredim \times \featuredim}$ if there is a 
		non-zero vector $\vx \in \mathbb{R}^{\featuredim} \setminus \{ \mathbf{0} \}$ such that $\mathbf{A} \vx = \lambda \vx$. },first={eigenvalue},text={eigenvalue} }
	
\newglossaryentry{eigenvector}{name={eigenvector}, description={An eigenvector of a matrix $\mathbf{A}$ is a 
		non-zero vector $\vx \in \mathbb{R}^{\featuredim} \setminus \{ \mathbf{0} \}$ such that $\mathbf{A} \vx = \lambda \vx$ 
	with some \gls{eigenvalue} $\lambda$.},first={eigenvector},text={eigenvector} }

\newglossaryentry{evd}{name={eigenvalue decomposition}, description={The task of computing the eigenvalues and corresponding eigenvectors of a matrix.},first={eigenvalue decomposition (EVD)},text={EVD} }

\newglossaryentry{gdmethods}{name={gradient-based method}, description={Gradient-based methods 
		are iterative algorithms for finding the minimum (or maximum) of a differentiable objective function 
		of a parameter vector. These algorithms construct a sequence of approximations to an optimal parameter 
		vector whose function value is minimal (or maximal). As their name indicates, gradient-based methods 
		use the gradients of the objective function evaluated during previous iterations to construct a new (hopefully) 
		improved approximation of an optimal parameter vector.},first={gradient-based methods},text={gradient-based methods} }

\newglossaryentry{sgd}{name={subgradient descent}, description={Subgradient descent is a generalization of \gls{gd} that is obtained 
	by using sub-gradients (instead of gradients) to construct local approximations of an objective function (such as the training error as a 
	function of the weights of a hypothesis).},first={subgradient descent},text={subgradient descent} }
	
\newglossaryentry{stochGD}{name={stochastic gradient descent}, description={Stochastic gradient descent is obtained from \gls{gd} by 
	replacing the gradient of the objective function by a noisy (or stochastic) estimate.},first={stochastic gradient descent (SGD)},text={SGD} }

\newglossaryentry{pca}{name={principal component analysis}, description={Principal component analysis 
		determines a given number of new features that are obtained by a linear transformation (map) of the 
		raw features. },first={principal component analysis (PCA)},text={PCA} }
	
\newglossaryentry{loss}{name={loss}, description={With a slight abuse of language, we use the term loss either for
		\gls{lossfunc} itself or for its value for a specific pair of \gls{datapoint} and \gls{hypothesis}.},first={loss},text={loss} }

\newglossaryentry{lossfunc}{name={loss function}, description={A loss function is a map 
		$$\lossfun: \featurespace \times \labelspace \times \hypospace \rightarrow \mathbb{R}_{+}: \big( \big(\featurevec,\truelabel\big), h\big) \mapsto  \loss{(\featurevec,\truelabel)}{h}$$ which assigns a pair consisting of a datapoint, with features $\featurevec$ 
		and label $\truelabel$, and a hypothesis $h \in \hypospace$ the non-negative real number $\loss{(\featurevec,\truelabel)}{h}$. 
		The loss value $\loss{(\featurevec,\truelabel)}{h}$ quantifies the discrepancy between the true label $\truelabel$ 
		and the predicted label $h(\featurevec)$. Smaller (closer to zero) values $\loss{(\featurevec,\truelabel)}{h}$ mean 
		a smaller discrepancy between predicted label and true label of a data point. Figure \ref{fig_loss_function} depicts 
		a loss function for a given data point, with features $\featurevec$ and label $\truelabel$, as a function of the 
		hypothesis $h \in \hypospace$.  },first={loss function},text={loss function} }

\newglossaryentry{decisiontree}{name={decision tree}, description={A decision tree is a flow-chart like representation of 
		a hypothesis map $h$. More formally, a decision tree is a directed graph which reads in the feature vector $\featurevec$ 
	  of a \gls{datapoint} at its root node. The root node then forwards the \gls{datapoint} to one of its children nodes based on some 
	  elementary test on the features $\featurevec$. If the receiving children node is not a leaf node, i.e., it has itself children nodes, 
	  it represents another test. Based on the test result, the \gls{datapoint} is further pushed to one of its neighbours. This testing and 
	  forwarding of the \gls{datapoint} is repeated until the \gls{datapoint} ends up in a leaf node (having no children nodes). 
	  The leaf nodes represent sets (decision regions) constituted by feature vectors $\featurevec$ that are mapped to 
	  the same function value $h(\featurevec)$.},first={decision tree},text={decision tree} }

\newglossaryentry{noniiddata}{name={non i.i.d.\ data}, description={A dataset that cannot be well modelled as 
		realizations of \gls{iid} \gls{rv}s.},first={non-i.i.d.\ data},text={non-i.i.d.\ data} }

\newglossaryentry{API} 
{
	name={Application Programming Interface (API)},
	description={An Application Programming Interface (API) is a particular set of rules and specifications that a software program can follow to access and make use of the services and resources provided by another particular software program that implements that API},
	first={Application Programming Interface (API)},
	text={API}
}

\newglossaryentry{hilbertspace}{name={Hilbert space},description={A Hilbert space is a linear vector space that is equipped with an inner product between pairs of vectors. One important example for a Hilbert spaces is the Euclidean spaces $\mathbb{R}^{\featuredim}$, for some dimension $\featuredim$, which consists of Euclidean vectors $\vu = \big(u_{1},\ldots,u_{\featurelen}\big)^{T}$ along with the inner product $\vu^{T} \vv$.},first={Hilbert space},text={Hilbert space}}

\newglossaryentry{sample}{name={sample},description={A finite sequence (list) of \gls{datapoint}s $\vz^{(1)},\ldots,\vz^{(\sampleidx)}$ that 
		is obtained or interpreted as the realizations of $\samplesize$ \gls{iid} \gls{rv}s with the common probability distribution $p(\vz)$.
		The length $\samplesize$ of the sequence is referred to as the \gls{samplesize}.},first={sample},text={sample}}
	
\newglossaryentry{samplesize}
{name=sample size,
	description={The number of individual \gls{datapoint}s contained in a dataset that is 
	obtained from realizations of \gls{iid} \gls{rv}s.},first={sample size},text={sample size}
}

\newglossaryentry{ann}
{name=artificial neural network,
	description={An artificial neural network is a graphical (signal-flow) representation of a map from 
		features of a \gls{datapoint} at its input to a predicted label at its output.},first={artificial neural network (ANN)},text={ANN}
}

\newglossaryentry{randomforest}
{name=random forest,
	description={A random forest is a set (ensemble) of different \gls{decisiontree}s. Each of these \gls{decisiontree}s is 
		obtained by fitting a perturbed copy of the original dataset.},first = {random forest}, text={random forest}
}

\newglossaryentry{bagging}{name={bagging},description={bagging (or ``bootstrap aggregation'') is a generic 
		technique to improve or robustify a given ML method. The idea is to use the bootstrap to generate 
		perturbed copy of a given training set and then apply the original ML method to learn a separate 
		hypothesis for each perturbed copy of the training set. The resulting set of hypotheses is then 
		used to predict the label of a \gls{datapoint} by combining or aggregating the individual predictions of each 
		individual hypothesis. For hypotheses that deliver numeric label values (regression methods) this 
		aggregation could be implemented by computing the average of individual predictions.},first={bootstrap aggreation (bagging)},text={bagging}}

\newglossaryentry{gd}{name={gradient descent},description={Gradient descent is an iterative method for finding the minimum 
		of a differentiable function $f(\weights)$. },first={gradient descent (GD)},text={GD}}

\newglossaryentry{ladregression}{name={least absolute deviation regression},description={Least absolute deviation regression uses 
		the average of the absolute precondition errors to find a linear hypothesis.},first={least absolute deviation regression},
	   text={least absolute deviation regression}}

\newglossaryentry{metric}{name={metric},description={A metric refers to a loss function that is used solely 
	    for the final performance evaluation of a learnt hypothesis. The metric is typically a \gls{lossfunc} that 
	    has a ``natural'' interpretation (such as the $0/1$ loss \eqref{equ_def_0_1}) but is not a good choice to guide 
	    the learning process, e.g., via \gls{erm}. For \gls{erm}, we typically prefer \gls{lossfunc}s that depend smoothly 
	    on the (parameters of the) hypothesis. Examples for such smooth loss functions include the squared error 
	    loss \eqref{equ_squared_loss} and the \gls{logloss} \eqref{equ_log_loss}.},first={metric},text={metric}}

\newglossaryentry{bayesrisk}{name={Bayes risk},description={We use the term Bayes risk as a synonym for the \gls{risk} or expected \gls{loss} 
		of a hypothesis. Some authors reserve the term Bayes risk for the \gls{risk} of a hypothesis that achieves minimum \gls{risk}, such a hypothesis 
		being referred to as a \gls{bayesestimator} \cite{LC}.},first={Bayes risk},text={Bayes risk}}
	
\newglossaryentry{bayesestimator}{name={Bayes estimator},description={A hypothesis whose \gls{bayesrisk} is minimal \cite{LC}.},first={Bayes estimator},text={Bayes estimator}}

\newglossaryentry{weights}{name={weights},
	description={We use the term weights synonymously for a finite set of parameters within a \gls{model}. 
		For example, the linear model consists of all linear maps $h(\featurevec)= \weights^{T} \featurevec$ 
		that read in a feature vector $\featurevec=\big(\feature_{1},\ldots,\feature_{\featurelen}\big)^{T}$ of a \gls{datapoint}. 
		Each specific linear map is characterized by specific choices for the parameters for weights $\weights = \big( \weight_{1},\ldots,\weight_{\featuredim}\big)^{T}$.},first={weights},text={weights}}

\newglossaryentry{parameters}{name={parameters},
	description={The parameters of a ML \gls{model} are tunable (learnable or adjustable) quantities that 
		allow to choose between different hypothesis maps. For example, the linear model $\hypospace \defeq \{h: h(\feature)= \weight_{1} \feature + \weight_{2}\}$ consists  of all hypothesis maps $h(\feature)= \weight_{1} \feature + \weight_{2}$ with a particular 
		choice for the parameters $\weight_{1},\weight_{2}$. Another example of parameters are the weights 
		assigned to the connections of an \gls{ann}.},first={parameters},text={parameters}}

\newglossaryentry{lln}{name={law of large numbers},
	description={The law of large numbers refers to the convergence of the average of an increasing number of 
		\gls{iid} \gls{rv}s to the mean (or expectation) of their common probability distribtuion.},first={law of large numbers},text={law of large numbers}}

\newglossaryentry{nonsmooth}{name={non-smooth},
	description={We refer to a function as non-smooth if it is not \gls{smooth} \cite{nesterov04}.},first={non-smooth},text={non-smooth}}

\newglossaryentry{convex}{name={convex},
	description={A set $\mathcal{C}$ in $\mathbb{R}^{\featuredim}$ is called convex if it contains the line segment 
		between any two points of that set. A function is called convex if its epigraph is a convex set \cite{BoydConvexBook}.},first={convex},text={convex}}

\newglossaryentry{smooth}{name={smooth},
	description={We refer to a real-valued function as smooth if it is differentiable and its gradient is continuous \cite{nesterov04,CvxBubeck2015}. In particular, a differentiable function $f(\weights)$ is $\beta$-smooth if the gradient $\nabla f(\weights)$ is Lipschitz continuous with Lipschitz constant $\beta$, i.e., $\| \nabla f(\weights) - \nabla f(\weights') \| \leq \beta \| \weights - \weights' \|$. },first={smooth},text={smooth}}

\newglossaryentry{dataug}{name={data augmentation},
	description={Data augmentation methods add synthetic \gls{datapoint}s to an existing set of \gls{datapoint}s. 
		These synthetic \gls{datapoint}s might be obtained by perturbations (adding noise) or 
		transformations (rotations of images) of the original \gls{datapoint}s.},first={data augmentation},text={data augmentation}}

\newcommand\defeq{:=}

\newcommand{\vx}[0]{{\bf x}}
\newcommand{\vv}[0]{{\bf v}}
\newcommand{\vu}[0]{{\bf u}}

\newcommand{\mA}[0]{{\bf A}}
\newcommand{\mL}[0]{{\bf L}}

\newcommand{\vw}[0]{{\bf w}}

\newcommand{\mQ}{\mathbf{Q}}

\newcommand{\va}[0]{{\bf a}}

\newcommand{\vz}[0]{{\bf z}}

\newcommand{\prob}[1]{p({#1})} 
 
\def \expect {\mathbb{E} }

\newcommand{\biasterm}{B}

\newcommand{\bmx}[0]{\begin{bmatrix}}
\newcommand{\emx}[0]{\end{bmatrix}}

\newcommand{\featuredim}{n}

\newcommand{\featurelen}{\featuredim}

\newcommand{\samplesize}{m}
\newcommand{\sampleidx}{i} 
 
\newcommand{\datapoint}{\vz} 
 
\newcommand{\clusteridx}{c} 

\newcommand{\nrcluster}{k} 
\newcommand{\nrseeds}{s} 
\newcommand{\featureidx}{j}

\newcommand\truelabel{y}

\newcommand\featurevec{\vx}
\newcommand\feature{x}

\newcommand\dataset{\mathcal{D}}

\newcommand\effdim[1]{d_{\rm eff} \left( #1 \right)}

\newcommand{\hypospace}{\mathcal{H}}

\newcommand{\featurespace}{\mathcal{X}}
\newcommand{\labelspace}{\mathcal{Y}}

\newcommand{\eigval}[1]{\lambda_{#1}}
\newcommand{\regparam}{\lambda}

\newcommand{\edges}{\mathcal{E}}

\newcommand{\loss}[2]{L\big({#1},{#2}\big)}

\DeclareMathOperator*{\argmin}{argmin}

\newcommand{\lossfun}{L}
\newcommand{\datacluster}[1]{\mathcal{C}^{(#1)}}

\newcommand{\weight}{w}
\newcommand{\weights}{\vw}

\newcommand{\decreg}[1]{\mathcal{R}_{#1}}

\newcommand{\nriter}{R}

\newcommand{\edgesigvec}{\mathbf{f}} 
\newcommand{\edgesig}{f} 
\newcommand{\edgeweight}{A}
\newcommand{\edgeweights}{\mA}

\newcommand{\graph}{\mathcal{G}}
\newcommand{\nodes}{\mathcal{V}}
\newcommand{\degreemtx}{\mathbf{D}}

\newcommand{\nodedegree}[1]{d^{(#1)}}
\newcommand{\nodeidx}{i}
\newcommand{\nrnodes}{n}

\newcommand{\edge}[2]{\{#1,#2\}}
\newcommand{\directededge}[2]{\left(#1,#2\right)}

\newcommand{\pditer}{r}

\newtheorem{theorem}{Theorem}

\usepackage{setspace}

\title{Flow-Based Clustering and Spectral Clustering: A Comparison} 

\name{Y. SarcheshmehPour, Y. Tian, L. Zhang, A. Jung}

\address{\normalsize Department of Computer Science, Aalto University, Finland; firstname.lastname(at)aalto.fi\\[-0.5mm] }

\begin{document}
	\maketitle
\begin{abstract}
We propose and study a novel graph clustering method for data with an intrinsic network structure.   
Similar to spectral clustering, we exploit an intrinsic network structure of data to 
construct Euclidean feature vectors. These feature vectors can then be fed into basic 
clustering methods such as k-means or Gaussian mixture model (GMM) based soft clustering. 
What sets our approach apart from spectral clustering is that we do not use the eigenvectors 
of a graph Laplacian to construct the feature vectors. Instead, we use the solutions of total 
variation minimization problems to construct feature vectors that reflect connectivity between 
data points. Our motivation is that the solutions of total variation minimization are 
piece-wise constant around a given set of seed nodes. These seed nodes can be obtained 
from domain knowledge or by simple heuristics that are based on the network structure of data. 
Our results indicate that our clustering methods can cope with certain graph structures 
that are challenging for spectral clustering methods.
\end{abstract}


\section{Introduction}

The analysis of networked data is often facilitated by grouping or clustering the data points 
into coherent subsets of data points. Clustering methods aim at finding subsets (clusters) of 
data points that are more similar to each other than to the remaining data points \cite{MLBasics}. 
Many basic clustering algorithms aim at clusters that are enclosed by a hypersphere or hyperellipsoid 
in a Euclidean feature space. These methods are most successful if data points are characterized by 
Euclidean feature vectors whose distance is small (large) for data points in the same (different) cluster(s). 

Graph clustering methods can be applied to data with an intrinsic network structure that reflects 
a notion of similarity between different data points. Many important application domains, ranging 
from the Internet of Things to the management of pandemics, generate distributed collections of 
local datasets (``big data over network'') \cite{BigDataNetworksBook,JuLiveProject2021}. The 
network structure of these local datasets might be  induced by spatio-temporal proximity (``contact networks''), 
statistical dependencies,  or functional relations \cite{graphic2009,hierarchical2015}. 

\section{Problem Formulation} 

We represent networked data using an undirected “empirical” graph $\graph=(\nodes, \edges, \edgeweights)$ \cite{SemiSupervisedBook,Luxburg2007}. 
Every node $\nodeidx \in \nodes = \{1,\ldots,\nrnodes\}$ of the empirical graph represents a datapoint. The dataset 
consists of $\nrnodes$ different datapoints. A datapoint might be a single sensor measurement, an entire time 
series,  or even a whole collection of videos. 

Our approach does not take the individual nature of datapoints into account. Rather, we only 
use their similarities as encoded in the weighted edges of $\graph$. In particular, two nodes $\nodeidx, \nodeidx'$ 
are connected by an edge $\edge{\nodeidx}{\nodeidx'}$ if the corresponding datapoints are similar. 
The amount of similarity is quantified by a positive edge weight $\edgeweight_{\nodeidx,\nodeidx'}$. 

We find it convenient to use an oriented (or directed) version of the empirical graph $\graph$ by 
declaring, for each undirected edge $\edge{\nodeidx}{\nodeidx'}$, the node $\min\{\nodeidx,\nodeidx'\}$ 
as tail and the other node $\max\{\nodeidx,\nodeidx'\}$ as head in the corresponding directed edge. The resulting 
oriented empirical graph $\overrightarrow{\graph} \defeq \big( \nodes, \overrightarrow{\edges}, \edgeweights \big)$ 
has the same nodes as $\graph$ and contains the directed edge $\directededge{\nodeidx}{\nodeidx'}$ from node $\nodeidx$ to node $\nodeidx'$ 
if and only if $\nodeidx < \nodeidx'$ and $\edge{\nodeidx}{\nodeidx'} \in \edges$. This directed edge $\directededge{\nodeidx}{\nodeidx'} \in \overrightarrow{\edges}$ has the same 
weight $\edgeweight_{\nodeidx,\nodeidx'}$ as the corresponding undirected edge $\edge{\nodeidx}{\nodeidx'} \in \edges$. 

We will use several matrices that are naturally associated with an empirical graph $\graph$. 
The weight matrix $\edgeweights$ contains the edge weight $\edgeweight_{\nodeidx,\nodeidx'}$ 
in its $\nodeidx$th row and $\nodeidx'$th column. The diagonal degree matrix $\degreemtx \defeq {\rm diag}\{\nodedegree{1},\ldots,\nodedegree{\nrnodes}\}$ collects the (weighted) node degrees 
$\nodedegree{\nodeidx}=\sum_{\nodeidx'} \edgeweight_{\nodeidx,\nodeidx'}$ for each node $\nodeidx$. 
The graph Laplacian matrix $\mathbf{L} \defeq \degreemtx - \edgeweights$ is instrumental 
for spectral clustering methods (see Section \ref{sec_spec_clustering}). 

There are also vector spaces naturally associated with an empirical graph. The vector space consisting of 
maps $\vu$ from nodes to real numbers is denoted $\mathbb{R}^{\nrnodes}$. A vector $\vu \in \mathbb{R}^{\nrnodes}$ 
is a function or map that assigns each node $\nodeidx \in \nodes$ a real-number $u_{\nodeidx}$. 
Similarly, we define the vector space $\mathbb{R}^{\overrightarrow{\edges}}$ that consists of vectors 
$\edgesigvec$ that assign each edge $\directededge{\nodeidx}{\nodeidx'} \in \edges$ a real-number $\edgesig_{\directededge{\nodeidx}{\nodeidx'}}$. 


%
%

Many graph clustering methods, such as spectral clustering, rely on a clustering assumption. 
This clustering assumption applies to a dataset for which we know a useful empirical graph 
$\graph$ 
 \cite{ Ng2001,  Spielma2012, Luxburg2007}. A cluster is then 
constituted by data points whose nodes in $\graph$ are more densely connected with each other 
than with nodes representing data points outside the cluster \cite{complex2013,complex2015}. 
This informal clustering assumption can be made precise by constraining the cut-size of cluster \cite{Luxburg2007}. 
By the maxflow/mincut duality, requiring a small cut is equivalent to requiring a minimum amount of network 
flow that can be routed from the nodes inside a cluster through its boundary \cite{Boykov2004}. 

Using the above clustering assumption, we construct feature vectors $\featurevec^{(\nodeidx)}$ for 
each node (or datapoint) $\nodeidx \in \nodes$. The feature vectors will be constructed such that 
two nodes $\nodeidx, \nodeidx'$ that belong to a well-connected subset of nodes (a cluster) have more 
feature vectors $\featurevec^{(\nodeidx)}, \featurevec^{(\nodeidx')}$ with a small Euclidean distance 
$\big\| \featurevec^{(\nodeidx)}- \featurevec^{(\nodeidx')} \big\|$. Loosely speaking, the feature construction 
maps similarity between datapoints in the empirical graph $\graph$ into proximity
of their Euclidean feature vectors. This allows us then to apply standard clustering methods, such as k-means 
of soft clustering \cite{MLBasics}, to the network structured data. 


\section{Spectral Clustering} 
\label{sec_spec_clustering} 

Before we detail our construction of the feature vectors $\featurevec^{(\nodeidx)}$ let us 
briefly review the construction used by spectral graph clustering methods \cite{Luxburg2007} . 
The most basic variant of spectral clustering constructs the node feature vectors $\featurevec^{(\nodeidx)}$ using 
the eigenvectors of the graph Laplacian matrix $\mathbf{L} = \degreemtx - \edgeweights$ \cite{Luxburg2007}. 
The matrix $\mathbf{L}$ is positive semi-definite (psd) and therefore we can find an orthonormal set of 
eigenvectors \cite{golub96}
 \begin{equation}
 \label{equ_def_eigvals}
\vu^{(1)},\ldots,\vu^{(\nrnodes)} \mbox{ with eigenvalues } 0 \leq \eigval{1} \leq \ldots \leq \eigval{\nrnodes}.
\end{equation} 
For a given a number $\nrcluster$ of clusters, spectral clustering methods use the feature vectors 
\begin{equation} 
\label{equ_def_feat_vec_spec_cluster}
\featurevec^{(\nodeidx)} = \big( u^{(1)}_{\nodeidx}, \ldots,u^{(\nrcluster)}_{\nodeidx} \big)^{T} \mbox{ for every node } \nodeidx \in \nodes. 
\end{equation} 

{\bf The Ideal Case.} To develop some intuition for the usefulness of the construction \eqref{equ_def_feat_vec_spec_cluster}, 
consider an empirical graph that contains $\nrcluster$ components $\datacluster{1},\ldots,\datacluster{\nrcluster} \subseteq \nodes$ 
that are not connected by any edge. In this case, the first $\nrcluster$ (possibly repeated) eigenvalues in \eqref{equ_def_eigvals} 
are equal to $0$. The correspond eigenvectors $\vu^{(1)},\ldots,\vu^{(\nrcluster)}$ constitute an orthonormal basis 
for the subspace of $\mathbb{R}^{\nrnodes}$ that is spanned by the indicator vectors $\mathbf{e}^{(\clusteridx)}$ 
of the components $\datacluster{\clusteridx}$, $e_{\nodeidx}^{(\clusteridx)} = 1$ for nodes $\nodeidx \in \datacluster{\clusteridx}$ and $e_{\nodeidx}^{(\datacluster{\clusteridx})} = 0$ for nodes $\nodeidx \notin \datacluster{\clusteridx}$. Thus, the feature vectors of nodes in the 
same cluster $\datacluster{\clusteridx}$ are identical and orthogonal to the feature vectors of nodes in different 
clusters $\datacluster{\clusteridx'}$ with $\clusteridx' \neq \clusteridx$. 

{\bf General Case.} In general, the empirical graph will contain edges between different clusters. 
A widely used approach to analyzing spectral clustering methods is to consider the effect of these 
inter-cluster edges as perturbations of the Laplacian matrix $\mL$. For sufficiently small perturbations, 
the subspace spanned by the first $\nrcluster$ eigenvectors of $\mL$ will be close to the subspace spanned 
by the indicator vectors $\mathbf{e}^{(\clusteridx)}$. This deviation can be made precise using tools 
from linear algebra and results in conditions on the empirical graph for the success of k-means 
when applied to the feature vectors \eqref{equ_def_feat_vec_spec_cluster} \cite{Luxburg2007,Ng2001}.  

There are certain types of empirical graphs that are challenging for spectral clustering methods that use 
the features \eqref{equ_def_feat_vec_spec_cluster} \cite{Luxburg2007,NIPS2006_bdb6920a,NSZ09}. 
For example, spectral methods tend to fail for datasets that consist of clusters with significantly 
varying sizes (see Section \ref{unsuper_experiment}). Another example of a challenging network structure for spectral 
clustering is a chain graph \cite{NNSPFrontiers2018}. 

Consider a chain graph that consists of two clusters connected via a single edge. The 
weights of all intra-cluster edges are equal to $1$ while the weight $\edgeweight_{o}$ of the single boundary 
edge is slightly smaller. As the chain graph is connected,  $\eigval{1}=0$ with corresponding eigenvector $\vu^{(1)}$ having identical entries. The smallest non-zero eigenvalue is $\eigval{2}$ has a corresponding 
eigenvector $\vu^{(2)}$ whose entries are depicted in Figure \ref{fig_solution_nLasso_chain}.
The resulting feature vectors (see \eqref{equ_def_feat_vec_spec_cluster} with $k=2$) do not reflect well the cluster structure of the chain graph.

\begin{figure}[htbp]
	\begin{center}
	\resizebox{7cm}{2.5cm}{
		\begin{tikzpicture}
			\tikzset{x=0.35cm,y=2.5cm,every path/.style={>=latex},node style/.style={circle,draw}}
			\csvreader[ head to column names,%
			late after head=\xdef\iold{\i}\xdef\xold{\x},,%
			after line=\xdef\iold{\i}\xdef\xold{\x}]%
			{nLassoChain.csv}{}
			{\draw [line width=0.0mm] (\iold, \xold) (\i,\x) node {\large $\circ$};
			}
			\csvreader[ head to column names,%
			late after head=\xdef\iold{\i}\xdef\xold{\x},,%
			after line=\xdef\iold{\i}\xdef\xold{\x}]%
			{FiedlerChain.csv}{}
			{\draw [line width=0.0mm] (\iold, \xold) (\i,\x) node {\large $\star$};
			}
			\draw[->] (0.4,0) -- (22,0);
			\node [right] at (19.5,0.1) {\centering node $i$};
			\draw[->] (0.5,0) -- (0.5,1.1);	
			\draw[-,dashed] (0,0.5) -- (22,0.5);
			\foreach \label/\labelval in {0/$0$,0.5/$0.5$,1/$1$}
			{ 
				\draw (0.4,\label) -- (0.6,\label) node[left=2mm] {\large \labelval};
			}
			\foreach \label/\labelval in {1/$1$,5/$5$,10/$10$,15/$15$,20/$20$}
			{ 
				\draw (\label,1pt) -- (\label,-2pt) node[below] {\large \labelval};
			}
		\end{tikzpicture}}
		\vspace*{-4mm}
	\end{center}
	\caption{Solution of TV minimization (``$\circ$'') for the chain graph 
		obtained from Algorithm \ref{alg1} using $\nriter=1000$ iterations.  Entries (``$\star$'') 
		of the eigenvector $\vu^{(2)}$ corresponding to the smallest non-zero eigenvalue $\eigval{2}$ (see \eqref{equ_def_eigvals}). 
	}
\label{fig_solution_nLasso_chain}
\vspace*{-3mm}
\end{figure}
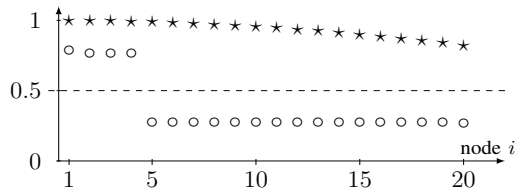

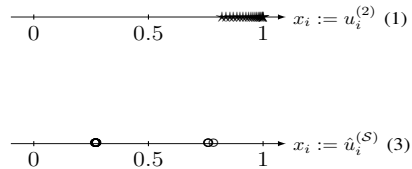
\begin{figure}[htbp]
	\begin{center}
		\resizebox{5.5cm}{2.3cm}{
		\begin{tikzpicture}
			\tikzset{x=4cm,y=2.5cm,every path/.style={>=latex},node style/.style={circle,draw}}
			\csvreader[ head to column names,%
			late after head=\xdef\iold{\i}\xdef\xold{\x},,%
			after line=\xdef\iold{\i}\xdef\xold{\x}]%
			{nLassoChain.csv}{}
			{\draw [line width=0.0mm] (\xold,0) (\x,0) node {\large $\circ$};
			}
			\csvreader[ head to column names,%
			late after head=\xdef\iold{\i}\xdef\xold{\x},,%
			after line=\xdef\iold{\i}\xdef\xold{\x}]%
			{FiedlerChain.csv}{}
			{\draw [line width=0.0mm] (\xold,1 ) (\x,1) node {\large $\star$};
			}
		    
		    \draw[->] (-0.1,0) -- (1.1,0);  
		    
		      \draw[->] (-0.1,1) -- (1.1,1);  
			
			\node [right] at (1.1,0.0) {$\feature_{\nodeidx} \defeq \hat{u}_{\nodeidx}^{(\seednodes)}$ \eqref{equ_def_nLasso}};
				\node [right] at (1.1,1.0) {$\feature_{\nodeidx} \defeq u_{\nodeidx}^{(2)}$ \eqref{equ_def_eigvals}};
			
			\foreach \label/\labelval in {0/$0$,0.5/$0.5$,1/$1$}
			{ 
				\draw (\label,1.02) -- (\label,0.98) node[below] {\large \labelval};
			}

			\foreach \label/\labelval in {0/$0$,0.5/$0.5$,1/$1$}
			{ 
				\draw (\label,0.02) -- (\label,-0.02) node[below] {\large \labelval};
			}
		\end{tikzpicture}}
		\vspace*{-4mm}
	\end{center}
	\caption{Scatterplot of (scalar) node features $\feature_{\nodeidx}$ constructed by spectral and flow-based clustering 
		for the chain graph in Figure \ref{fig_solution_nLasso_chain}. Spectral clustering uses the entries (``$\star$'') 
		of the eigenvector $\vu^{(2)}$ (see \eqref{equ_def_eigvals}). Flow-based clustering uses the solution of TV 
		minimization (``$\circ$'') (see \eqref{equ_def_nLasso}).
	}
\label{fig_scatterplot}
\vspace*{-3mm}
\end{figure}

We next describe our novel construction of feature vectors. The idea is to replace the eigenvectors 
of the Laplacian in \eqref{equ_def_feat_vec_spec_cluster} with solutions of total variation (TV) minimization problems. 
It turns out that these feature vectors better reflect the cluster geometry of $\graph$ in terms of 
bottlenecks for network flows between different clusters. These flow bottlenecks are the boundaries 
between the clusters that are obtained by applying clustering methods to these feature vectors.  





\section{Total Variation Minimization}
\label{sec_local_graph_clustering}

The construction of feature vectors \eqref{equ_def_feat_vec_spec_cluster} in spectral clustering 
methods is global. Indeed, the eigenvectors \eqref{equ_def_eigvals} of the graph Laplacian will 
typically depend on every edge in the empirical graph. Instead, we construct features in a more 
local fashion by considering local clusters around seed nodes. These seed nodes might be obtained 
by domain knowledge (telling us which nodes must belong to the same cluster), based on basic connectivity 
properties (nodes with high degree),  or chosen randomly. 

Given a set of seed nodes $\seednodes$ we solve the TV minimization problem 
\begin{equation}
\label{equ_def_nLasso}
\hat{\vu}^{(\seednodes)} = \argmin_{\vu \in \mathbb{R}^{\nodes}}  \sum_{\nodeidx \in \seednodes} (1\!-\!u_{\nodeidx})^2/2\!+\! \sum_{\nodeidx \notin \seednodes} (\alpha/2) u_{\nodeidx} ^2\!+\!\regparam \| \vu \|_{\rm TV}. 
\end{equation}
Here, we used the TV of a vector $\vu \in \mathbb{R}^{\nrnodes}$,
\begin{equation} 
\label{eq_def_TV}
\| \vu \|_{\rm TV} = \sum_{\edge{\nodeidx}{\nodeidx'} \in \edges} \edgeweight_{\nodeidx,\nodeidx'} |u_{\nodeidx} - u_{\nodeidx'}|.
\end{equation}
Note that the TV minimization \eqref{equ_def_nLasso} is parametrized by the set $\seednodes \subseteq \nodes$ of 
seed nodes. Our clustering methods solve multiple instances of \eqref{equ_def_nLasso}, each time using another choice for the seed nodes $\seednodes$. 
Section \ref{sec_flow_based_graph_clustering} discusses approaches for choosing a useful set of seed nodes. 

We have recently explored the duality between TV minimization \eqref{equ_def_nLasso} and network flow optimization \cite{JungYasmin2021}. 
This duality allows to characterize the solution $\hat{\vu}$ of \eqref{equ_def_nLasso} in terms of network flows. 
A network flow is a vector $\edgesigvec \in \mathbb{R}^{\overrightarrow{\edges}}$ whose entries $\edgesig_{\directededge{\nodeidx}{\nodeidx'}}$ 
represent a flow from node $\nodeidx$ to $\nodeidx'$. 
\begin{theorem}
\label{th_flow_characterization_nLasso}
A vector $\widehat{\vu}$ solves TV minimization \eqref{equ_def_nLasso} if and only if 
there is a flow vector $\edgesigvec \in \mathbb{R}^{\overrightarrow{\edges}}$ such that 
\begin{align} 
	-\sum_{\directededge{\nodeidx}{\nodeidx'} \in \overrightarrow{\edges}} \edgesig_{\directededge{\nodeidx}{\nodeidx'}} + \sum_{\directededge{\nodeidx'}{\nodeidx} \in \overrightarrow{\edges}}\edgesig_{\directededge{\nodeidx'}{\nodeidx}} & =  \hat{u}_{\nodeidx}\!-\!1 \mbox{ for } \nodeidx  \in \seednodes, \nonumber \\[2mm]
	-\sum_{\directededge{\nodeidx}{\nodeidx'} \in \overrightarrow{\edges}} \edgesig_{\directededge{\nodeidx}{\nodeidx'}}  
	+\sum_{\directededge{\nodeidx'}{\nodeidx} \in \overrightarrow{\edges}}\edgesig_{\directededge{\nodeidx'}{\nodeidx}}& =  \alpha \hat{u}_{\nodeidx}  \mbox{ for } \nodeidx  \notin \seednodes,  \nonumber \\[2mm]
	|\edgesig_{\directededge{\nodeidx}{\nodeidx'}}|  \leq \lambda \edgeweight_{\nodeidx,\nodeidx'}  \mbox{ for all } &\directededge{\nodeidx}{\nodeidx'} \in \overrightarrow{\edges}, \nonumber \\[2mm] 
     \hat{u}_{\nodeidx}\!-\!\hat{u}_{\nodeidx'}\!=\!0 \mbox{ for all } \directededge{\nodeidx}{\nodeidx'} \in \overrightarrow{\edges} \mbox{ with } & |\edgesig_{\directededge{\nodeidx}{\nodeidx'}}|\!<\!\lambda\edgeweight_{\nodeidx,\nodeidx'} . \label{equ_pd_optimal_non_staturated} 
\end{align} 
\end{theorem} 
\begin{proof}
The result can be obtained by applying Fenchel duality \cite[Ch. 31]{RockafellarBook} to TV minimization \eqref{equ_def_nLasso} 
and a dual minimum cost flow problem (see \cite[Sec. 3]{JungYasmin2021}). 
\end{proof} 

Let us illustrate Theorem \ref{th_flow_characterization_nLasso} for the empirical graph in Figure \ref{fig_solution_nLasso_chain}. 
This empirical graph is a chain graph and partitioned into two clusters $\datacluster{1}$ and $\datacluster{2}$ which 
are connected by a boundary edge $b$ with weight $\edgeweight_{o}$. 

{\bf Ideal Case.} Assume we solve TV minimization \eqref{equ_def_nLasso} for the chain graph in Figure \ref{fig_solution_nLasso_chain} 
with a single seed node $\seednodes = \{ \nodeidx^{(1)} \}$ with $\nodeidx^{(1)} \in \datacluster{1}$. Moreover, we choose 
$\regparam$ and $\alpha$ in \eqref{equ_def_nLasso} such that 
\begin{equation} 
\label{equ_cond_alpha_lambda}
\regparam \edgeweight_{o} < 1\mbox{, and } | \datacluster{1}| (\alpha/\regparam)+ \edgeweight_{o}<1, | \datacluster{2}| \alpha \geq \regparam \edgeweight_{o}. 
\end{equation} 
A direct inspection of the optimality condition \eqref{equ_pd_optimal_non_staturated} then yields the TV minimization solution   
\begin{equation} 
\label{equ_def_TV_min_chain}
\hat{u}_{\nodeidx} = \begin{cases} (1- \regparam\edgeweight_{o} )/ \big(1+ \alpha (|\datacluster{1}|-1) \big) & \mbox{ for } \nodeidx \in \datacluster{1} \\ 
	\regparam\edgeweight_{o} /  \big( \alpha |\datacluster{2}| \big) & \mbox{ for } \nodeidx \in \datacluster{2} \end{cases}.
\end{equation} 
Note that the vector \eqref{equ_def_TV_min_chain} is piece-wise constant over the clusters $\datacluster{1}$ and $\datacluster{2}$. 
Thus, if we would use \eqref{equ_def_TV_min_chain} as (single) feature $\feature_{\nodeidx} = \hat{u}_{\nodeidx}$, basic clustering 
methods would successfully recover the clusters $\datacluster{1}$ and $\datacluster{2}$. 

{\bf General Case.} Note that \eqref{equ_def_TV_min_chain} characterizes the solution of TV minimization \eqref{equ_def_nLasso} 
only when $\regparam$ and $\alpha$ are chosen such that \eqref{equ_cond_alpha_lambda} is valid. However,  condition \eqref{equ_cond_alpha_lambda} 
is not useful in practice as it involves the size of the clusters which we would like to determine. A practical approach to choosing $\alpha$ and $\lambda$ 
can be based on probabilistic models for the empirical graph such as stochastic block models \cite{JuPLSBMAsiloma2020}. Alternatively, the 
choice for $\alpha$ and $\lambda$ can be guided by the size of the piece $\{ \nodeidx \in \nodes: \hat{u}_{\nodeidx} = \hat{u}_{\nodeidx^{(1)}}\}$ 
around the seed node $\nodeidx^{(1)}$. This piece grows for increasing $\regparam$ and becomes the entire node set $\nodes$ whenever 
$\regparam \edgeweight_{\nodeidx,\nodeidx'} > 1 \mbox{ for all edges } \edge{\nodeidx}{\nodeidx'} \in \edges.$

The duality between TV minimization \eqref{equ_def_nLasso} and network flow optimization 
is also instrumental for developing iterative methods to solve \eqref{equ_def_nLasso}. 
Algorithm \ref{alg1} summarizes the application of a generic primal-dual method to solve \eqref{equ_def_nLasso} \cite{JungYasmin2021}. 

\begin{algorithm}[htbp]
	\caption{Primal-Dual Method For TV Minimization \eqref{equ_def_nLasso}}
	\label{alg1}
	{\bf Input}: $\graph=\big(\nodes,\edges,\edgeweights\big), \seednodes, \lambda, \alpha,  \nriter$\\
	{\bf Initialize}: $\hat{\vu}^{(0)}\!\defeq\!{\bf 0}$; $\hat{\vu}^{(-1)}\!\defeq\!{\bf 0}$;$\hat{\edgesigvec}^{(0)} \!\defeq\! {\bf 0}$; \\
	{\bf Output}: approximation $\widehat{\vu}$ of solution $\widehat{\vu}^{(\seednodes)}$ to \eqref{equ_def_nLasso} \\
	\begin{algorithmic}[1]
		\For{$\pditer = 0, \ldots,\nriter-1$}
		
		\State  \hspace*{-1mm}$\mbox{ {\bf for} } i\!\in\!\nodes \mbox{ {\bf do} } : \tilde{u}_{\nodeidx}  \!\defeq\! 2 \hat{u}^{(\pditer)}_{\nodeidx} - \hat{u}^{(\pditer\!-\!1)}_{\nodeidx} $

		\For{$ e = \directededge{\nodeidx}{\nodeidx'} \in \overrightarrow{\edges} $}
		\State $	\hat{\edgesig}^{(\pditer\!+\!1)}_{e} \!\defeq\!\hat{\edgesig}^{(\pditer)}_{e}\!+\! (1/2)  (\tilde{u}_{\nodeidx}\!-\!\tilde{u}_{\nodeidx'})$
		\State $\hat{\edgesig}^{(\pditer\!+\!1)}_{e} \!\defeq\! \hat{\edgesig}_{e}^{(\pditer\!+\!1)}\!/\!\max\{1, |\hat{\edgesig}_{e}^{(\pditer\!+\!1)}|/(\regparam \edgeweight_{e}) \} $ \label{equ_step_cap_constraint}
		\EndFor

		\For{$ \nodeidx \in \nodes $}
		\State $\hat{u}^{(\pditer\!+\!1)}_{\nodeidx}  \!\defeq\! \hat{u}^{(\pditer)}_{\nodeidx}\!-\!\gamma_{\nodeidx} \bigg[\hspace*{-1mm}\sum_{\directededge{\nodeidx}{\nodeidx'}} \hat{\edgesig}^{(\pditer\!+\!1)}_{\directededge{\nodeidx}{\nodeidx'}}  \!-\!\hspace*{-1mm}\sum_{\directededge{\nodeidx'}{\nodeidx}} \hat{\edgesig}^{(\pditer\!+\!1)}_{\directededge{\nodeidx'}{\nodeidx}} \bigg]$ \label{equ_compensate_signal_value}
		
		\If {$\nodeidx \in \seednodes$}
		\State $\hat{u}_{\nodeidx}^{(\pditer\!+\!1)} \!\defeq\!  \big(\gamma_{\nodeidx}\!+\!\hat{u}^{(\pditer\!+\!1)}_{\nodeidx}\big)/(\gamma_{\nodeidx}\!+\!1)$ \label{equ_step_inject_flow_seed_nodes}
		\Else
		\State $\hat{u}_{\nodeidx}^{(\pditer\!+\!1)} \!\defeq\! \hat{u}^{(\pditer\!+\!1)}_{\nodeidx}/(\alpha\gamma_{\nodeidx}\!+\!1)$ \label{equ_step_leakage_flow_non_seed}
		\EndIf
		
		\EndFor
		
		\EndFor
		\State $\widehat{\vu} \defeq \hat{\vu}^{(\nriter)}$
	\end{algorithmic}
\end{algorithm}
The output $\widehat{\vu} \in \mathbb{R}^{\nrnodes}$ of Algorithm \ref{alg1} is an approximation to 
the solution of \eqref{equ_def_nLasso}. We assume that the number of iterations $\nriter$ used for 
Algorithm \ref{alg1} is sufficiently large such that the output of Algorithm \ref{alg1} can be considered 
a (numeric) solution to \eqref{equ_def_nLasso}. The number $\nriter$ of iterations can be guided by 
probabilistic models for the underlying empirical graph combined with the convergence rates guaranteed 
by primal-dual methods \cite{PrecPockChambolle2011}. Alternatively, we can tune the number of 
iterations based on the final clustering result obtained by using Algorithm \ref{alg1} as a sub-routine 
within our clustering method (see Algorithm \ref{unsuperalg}). 

It is instructive to interpret Algorithm \ref{alg1} as a message passing method for iteratively optimizing the network flow 
$\hat{\edgesig}^{(\pditer)}_{\directededge{\nodeidx}{\nodeidx'}}$. Step \ref{equ_step_cap_constraint} enforces 
the capacity constraint $| \hat{\edgesig}^{(\pditer)}_{\directededge{\nodeidx}{\nodeidx'}} | \leq \regparam \edgeweight_{\nodeidx,\nodeidx'}$. 
Step \ref{equ_compensate_signal_value} adjusts the value $\hat{u}^{(\pditer)}_{\nodeidx}$ based on the 
net flow into the node $\nodeidx$. In step \ref{equ_step_inject_flow_seed_nodes}, flow is injected into 
seed nodes $\nodeidx \in \seednodes$ while in step \eqref{equ_step_leakage_flow_non_seed} flow is leaked out of remaining nodes $\nodeidx \notin \seednodes$.

\section{Flow-Based Graph Clustering}
\label{sec_flow_based_graph_clustering} 

We are now in the position to formulate our flow-based graph clustering method as Algorithm \ref{unsuperalg}. 
This method constructs feature vectors $\featurevec$ using the solutions of TV minimization \eqref{equ_def_nLasso} 
for different choices of seed nodes $\seednodes$. Instead of using the values of eigenvectors 
for the graph Laplacian (as used by spectral clustering), we use entries of the vector solving TV minimization \eqref{equ_def_nLasso} 
to construct feature vectors for each node $\nodeidx \in \nodes$. 

\begin{algorithm}[htbp]
	\caption{Flow-Based Graph Clustering}
	\label{unsuperalg}
	{\bf Input}: empirical graph $\graph$, TV min. parameters $\regparam$, $\alpha$, 
	number $\nrcluster$ of cluster; number $\nrseeds$ of seeds.\\
	{\bf Output}: cluster assignments $\hat{c}_{1},  \ldots , \hat{c}_{\nrnodes} \in \{1,\ldots,\nrcluster\}$ \\
	\begin{algorithmic}[1]
		\For{$\pditer = 1, \ldots,\nrseeds$}
		\State select new seed nodes $\seednodes$  \label{equ_step_select_seed_nodes} with Algorithm \ref{selseednodes}
		\State run Algorithm \ref{alg1} with $\graph,\seednodes,\regparam,\alpha$
		\State store resulting vector in $\widehat{\vu}^{(\pditer)}$
		\EndFor
	\State construct node features 
\begin{equation} 
	\label{equ_def_flow_vec_features}
	\featurevec^{(\nodeidx)} = \big( \hat{u}^{(1)}_{\nodeidx}, \ldots,\hat{u}^{(\nrseeds)}_{\nodeidx} \big)^{T} \mbox{ for every node } \nodeidx \in \nodes. 
\end{equation} 
 \State compute cluster assignments $\hat{c}_{\nodeidx}$ by applying k-means 
 to feature vectors \eqref{equ_def_flow_vec_features} 
	\end{algorithmic}
\end{algorithm}
A key challenge for the successful application of Algorithm \ref{unsuperalg} is a suitable 
section of seed nodes in step \ref{equ_step_select_seed_nodes}. One simple 
approach is to choose the set $\seednodes$ by randomly selecting a single 
node $\nodeidx \in \nodes$. In general it is preferable to use more than one seed node. 
However,  we must ensure that a particular selection of seed nodes $\seednodes$ 
contains only nodes from the same cluster. 

Our numerical experiments (see Section \ref{sec_num_experiment}) use a simple  
heuristic method for selecting seed nodes $\seednodes$ in step \eqref{equ_step_select_seed_nodes} of Algorithm \ref{unsuperalg}.
This heuristic method is summarized in Algorithm \ref{selseednodes}. Algorithm \ref{selseednodes} constructs 
a new set of seed nodes by first randomly choosing a node $\nodeidx'$ with large degree $\nodedegree{\nodeidx}$ and 
then adding all neighbours of $\nodeidx'$ that have a sufficiently large number of common neighbours with $\nodeidx'$.   

\begin{algorithm}[htbp]
	\caption{Select Seed Nodes $\seednodes$}
	\label{selseednodes}
	{\bf Input}: empirical graph $\graph$, minimum number of common neighbours $\eta$, minimum degree $d$\\
	{\bf Output}: seed nodes $\seednodes \subseteq \nodes$ \\
	\begin{algorithmic}[1]
	   \State initialize $\seednodes \defeq \emptyset$ 
		\State determine  $\nodeidx' \defeq randomly\ select\ from\ \{\nodeidx : \nodedegree{\nodeidx} \geq d\}$
		\State add node, $\seednodes \defeq \{ \nodeidx' \}$ 
		\For{each $\nodeidx''$ with $\edge{\nodeidx'}{\nodeidx''} \in \edges$}
		\State $\mathcal{D} = \{ \nodeidx''': \edge{\nodeidx'''}{\nodeidx'}, \edge{\nodeidx''}{\nodeidx'''} \in \edges \}$
		\If{$| \mathcal{D} | \geq \eta$}
		\State $\seednodes \defeq \seednodes \bigcup \{ \nodeidx'' \}$
		\EndIf
		\EndFor
	\end{algorithmic}
\end{algorithm}

\section{Numerical Experiments}
\label{sec_num_experiment}

For datasets without intrinsic graphs, we construct the empirical graph by using an Euclidean distance based equation to define the similarity between data points:  $A_{i, j}:=\exp \left(-\left\|\mathbf{x}^{(i)}-\mathbf{x}^{(j)}\right\|_{2}^{2} /\left(2 \sigma^{2}\right)\right)$.
Here, $\sigma$ is a tuning parameter that chosen via cross-validation or using a probabilistic model for the data points to keep numerical stability.

\subsection{Experiments with Synthetic data}
\label{unsuper_experiment}


A dataset \cite{NIPS2006_bdb6920a} with two clusters, as depicted in Figure \ref{fig_gu}, the first cluster is a set of 2D data points drawn from Gaussian density centered at $(2,0.2)$ with diagonal covariance matrix $ 0.01*I $, the second cluster is a set of 2D data points denoting uniform density in a rectangular region: $\left\{ ( x_{1}, x_{2})  \mid 0 < x_{1}< 8, {-0.05} < x_{2} < 0 \right\}$. As shown in Figure \ref{fig_gu}, our method significantly outperforms spectral clustering. Parameters used for Algorithm \ref{unsuperalg}  are $\alpha = 0.005$, $\lambda=0.01$,  $\eta=50$,  and $d=12$.\\


\begin{figure}[htbp]
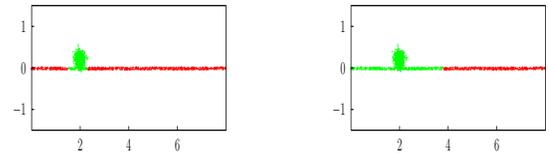

	\begin{center}
 }
	
	\vspace*{-4mm}
	\end{center}
	\caption{ The left plot is the clustering result of Algorithm \ref{unsuperalg}, the right one is the clustering result of spectral clustering \cite{Ng2001}.
	}
\label{fig_gu}
\vspace*{-3mm}
\end{figure}

\subsection{Image Segmentation / Pixel Clustering}
\label{segmentation}

The performance of Algorithm \ref{unsuperalg} for image segmentation/pixel clustering is tested on 
some RGB images. We construct an empirical graph as mentioned in \ref{sec_num_experiment} based on pixel values.  Each pixel is connected to other pixels within up to three hops.  
empirical graph is forced to be sparse by removing some edges associated with small weights. 
Pixels in a rectangular are set to be seed nodes. After 1000 iterations, a local cluster which 
segment out the object of interest from the background is determined. As depicted in fig \ref{fig:bear}, 
the algorithm can accurately detect which pixels belong to the object. To compare and contrast, 
we also use spectral clustering to execute this task, the result in fig \ref{fig:bear} shows that 
our algorithm out-performs spectral clustering.\\
The source code for the above experiments can be found at \url{https://github.com/YuTian8328/}. 

\begin{figure}
    \centering
 \includegraphics[width=2.7cm, height=1.8cm]{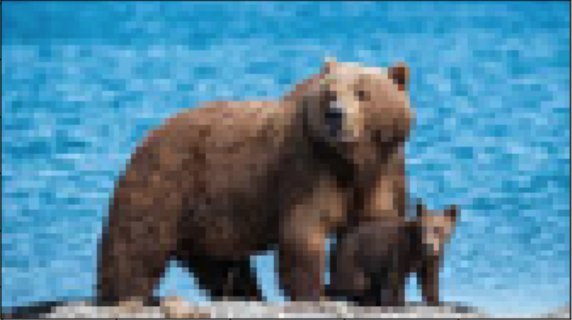}
 \includegraphics[width=2.7cm, height=1.8cm]{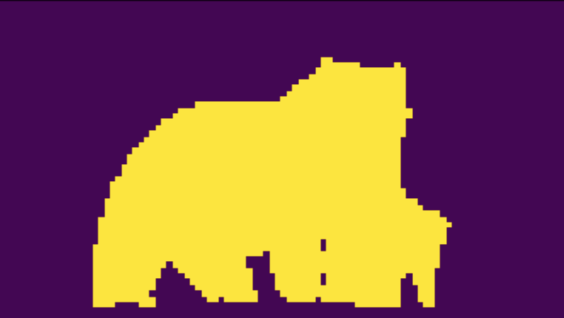}
  \includegraphics[width=2.7cm, height=1.8cm]{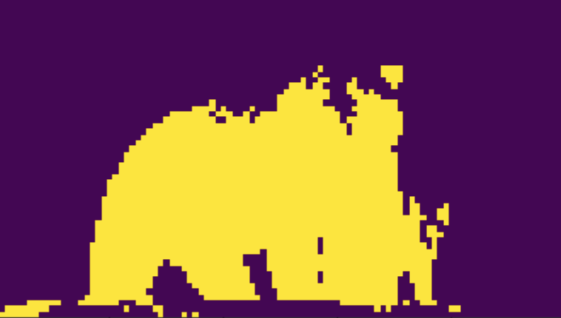}
    \caption{Image Segmentation: The left plot is the original image, the middle one is the segmentation 
    	result from Algorithm \ref{unsuperalg}, the right one is the result from spectral clustering.}
\label{fig:bear}
\end{figure}


\bibliographystyle{IEEEbib}
\bibliography{Asilomar2021_nLassoClustering.bbl}

\begin{thebibliography}{1}



\bibitem{graphic2009}
D.~Koller and N.~Friedman.
\newblock ``Probabilistic Graphical Models: Principles and Techniques. Adaptive computation and machine learning,``
\newblock in {\em MIT Press}, 2009.



\bibitem{hierarchical2015}
S.~Banerjee, B.P.~Carlin, and A.E.~Gelfand.
\newblock ``Hierarchical Modeling and Analysis for Spatial Data,``
\newblock in {\em Chapman and Hall/CRC}, 2015.



\bibitem{social2015}
H.-J.~Li and J.J.~Daniels.
\newblock ``Social significance of community structure: Statistical view,``
\newblock in {\em IEEE Phys. Rev}, 91(1):012801, Jan. 2015.



\bibitem{PrecPockChambolle2011}
T. Pock and A. Chambolle.
\newblock ``Diagonal preconditioning for first order primal-dual algorithms in convex optimization,``
\newblock in {\em IEEE ICCV}, 21, Nov. 2011.



\bibitem{JungDualitynLasso}
A.~Jung.
\newblock ``On the duality between network flows and network lasso,``
\newblock {\em IEEE Sig. Proc. Lett.}, 27:940 -- 944, 2020.



\bibitem{Ng2001}
A.~Y. Ng, M.~I. Jordan, and Y.~Weiss,
\newblock ``On spectral clustering: Analysis and an algorithm,``
\newblock in {\em Adv. Neur. Inf. Proc. Syst}, 2001.



\bibitem{Spielma2012}
D.~Spielman,
\newblock ``Spectral graph theory,``
\newblock in {\em U. Naumann and O. Schenk,
editors, Combinatorial Scientific Computing. Chapman and Hall/CRC}, 2012.



\bibitem{RockafellarBook}
R. T. Rockafellar,
\newblock ``Convex Analysis,``
\newblock in {\em Princeton Univ. Press}, 1970.



\bibitem{Luxburg2007}
U.~von Luxburg,
\newblock ``A tutorial on spectral clustering,``
\newblock {\em Statistics and Computing}, vol. 17, no. 4, pp. 395--416, Dec.
  2007.


\bibitem{JungYasmin2021}
A.~Jung and Y.~SarcheshmehPour,
\newblock ``Local Graph Clustering With Network Lasso,``
\newblock IEEE Signal Processing Letters ( Volume: 28), 2020, pp. 106--110.


\bibitem{golub96}
G. H. Golub and C. F. {Van Loan},
\newblock ``Matrix Computations, 3rd.``
\newblock Johns Hopkins University Press, 1996.



\bibitem{NetworkLasso}
D.~Hallac, J.~Leskovec, and S.~Boyd,
\newblock ``Network lasso: Clustering and optimization in large graphs,``
\newblock in {\em Proc. SIGKDD}, 2015, pp. 387--396.

\bibitem{BoydConvexBook}
S.~Boyd and L.~Vandenberghe,
\newblock ``Convex Optimization,``
\newblock Cambridge Univ. Press, Cambridge, UK, 2004.


\bibitem{JungTVMin2019}
A.~Jung, A~O. Hero, A.~Mara, S.~Jahromi, A.~Heimowitz, and Y.C. Eldar.
\newblock ``Semi-supervised learning in network-structured data via total
  variation minimization,``
\newblock {\em IEEE Trans. Signal Processing}, 67(24), Dec. 2019.



\bibitem{complex2013}
A.~BertrandMarc and M.~Moonen.
\newblock ``Seeing the bigger picture: How nodes can learn their place within a complex ad hoc network topology,``
\newblock {\em IEEE Signal Processing Magazine}, 30(3):71-82, May 2013.



\bibitem{survey2011}
M.C.V.~Nascimento and A.C.~De Carvalho.
\newblock ``Spectral methods for graph clustering–a survey,``
\newblock {\em European Journal of Operational Research}, 211(2):221–231, 2011.


\bibitem{JungNguyen2019}
A.~Jung and N.~Tran.
\newblock ``Localized linear regression in networked data,``
\newblock {\em IEEE Signal Processing Letters}, vol.26, no. 7, pp. 1090–1094, 2019.


\bibitem{complex2015}
B.~Saha, A.~Mandal, S.B.~Tripathy , D.~Mukherjee1
\newblock ``Complex Networks, Communities and Clustering: A survey,``
\newblock {\em CoRR}, vol. abs/1503.06277, 2015.



\bibitem{He2014}
B.~He, Y.~You, and X.~Yuan.
\newblock ``On the convergence of primal-dual hybrid gradient algorithm,``
\newblock {\em SIAM J. Imaging Sci.}, 7(4):2526--2537, 2014.



\bibitem{SBM_paper}
C.~Lee and D.J. Wilkinson .
\newblock ``A review of stochastic block models and extensions for graph clustering,``
\newblock {\em Applied Network Science.}, volume 4, Article number: 122, 2019



\bibitem{JuPLSBMAsiloma2020}
A.~Jung.
\newblock ``Clustering  in  partially  labeled  stochasticblock models via total variation minimization,``
\newblock {\em inProc.54th Asilomar Conf.},  Signals, Systems, Computers, Pa-cific Grove, CA, Nov. 2020.


\bibitem{BigDataNetworksBook}
S. Cui and A. Hero and Z.-Q. Luo and J.M.F. Moura.
\newblock ``Big Data over Networks,``
\newblock {\em cup.},  Cambridge, UK. 2016.


\bibitem{JuLiveProject2021}
A. Jung.
\newblock ``Federated Learning over Networks for Pandemics,``
\newblock {\em Manning.},  2021.



\bibitem{SemiSupervisedBook}
A . Chapelle and B. Sch{\"o}lkopf and A. Zien.
\newblock ``Semi-Supervised Learning,``
\newblock {\em The MIT Press.}, Cambridge, Massachusetts. 2006.



\bibitem{Boykov2004}
Y. Boykov and V. Kolmogorov.
\newblock ``An experimental comparison of min-cut/max-flow algorithms for energy minimization in vision,``
\newblock {\em IEEE Trans. Pattern Anal. Mach. Intell.}, vol. 26, no. 9, 2004.



\bibitem{NNSPFrontiers2018}
A. Jung and M. Hulsebos.
\newblock ``The Network Nullspace Property for Compressed Sensing of Big Data over Networks,``
\newblock {\em Front. Appl. Math. Stat.}, 2018.


\bibitem{NSZ09}
B. Nadler and N. Srebro and X. Zhou.
\newblock ``Statistical Analysis of Semi-Supervised Learning: The Limit of Infinite Unlabelled Data,``
\newblock {\em Advances in Neural Information Processing Systems 22.}, pages. 1330--1338, 2009.


\bibitem{NIPS2006_bdb6920a}
Nadler, Boaz and Galun, Meirav.
\newblock ``Fundamental Limitations of Spectral Clustering,``
\newblock {\em Advances in Neural Information Processing Systems.}, B. Sch\"{o}lkopf and J. Platt and T. Hoffman, 2007, vol. 19, MIT Press.


\bibitem{MLBasics}
A.~Jung.
\newblock ``Machine Learning: The Basics,``
\newblock {\em ArXiv e-prints https://arxiv.org/abs/1805.05052.}, 2020.




\end{thebibliography}

\end{document}